\documentclass[10pt,twocolumn,letterpaper]{article}

\usepackage[pagenumbers]{cvpr} % To force page numbers, e.g. for an arXiv version

\definecolor{cvprblue}{rgb}{0.21,0.49,0.74}
\usepackage[pagebackref,breaklinks,colorlinks,allcolors=cvprblue]{hyperref}
\usepackage{booktabs}
\usepackage{multirow}
\usepackage{pgfplots} \pgfplotsset{compat=1.18}
\usepackage[table]{xcolor}
\newcolumntype{A}{>{\columncolor{blue!10}}c}

\usepackage{amsthm}

\newtheorem{theorem}{Theorem}
\newtheorem{lemma}{Lemma}

\def\paperID{21956} % *** Enter the Paper ID here
\def\confName{CVPR}
\def\confYear{2026}

\title{Single-Round Scalable Analytic Federated Learning}

\author{
Alan T. L. Bacellar\textsuperscript{1}, \space
Mustafa Munir\textsuperscript{1}, \space
Felipe M. G. França\textsuperscript{3},\\
Priscila M. V. Lima\textsuperscript{2}, \space
Radu Marculescu\textsuperscript{1}, \space
Lizy K. John\textsuperscript{1}\\[4pt]
\textsuperscript{1}University of Texas at Austin \\
\textsuperscript{2}Federal University of Rio de Janeiro\\
\textsuperscript{3} Instituto de Telecomunicações, Porto (now at Google)\\[4pt]
{\tt\small alanbacellar@utexas.edu}
}

\begin{document}
\maketitle

\begin{abstract}
Federated Learning (FL) is plagued by two key challenges: high communication overhead and performance collapse on heterogeneous (non-IID) data. Analytic FL (AFL) provides a single-round, data distribution invariant solution, but is limited to linear models. Subsequent non-linear approaches, like DeepAFL, regain accuracy but sacrifice the single-round benefit. In this work, we break this trade-off. We propose SAFLe, a framework that achieves scalable non-linear expressivity by introducing a structured head of bucketed features and sparse, grouped embeddings. We prove this non-linear architecture is mathematically equivalent to a high-dimensional linear regression. This key equivalence allows SAFLe to be solved with AFL's single-shot, invariant aggregation law. Empirically, SAFLe establishes a new state-of-the-art for analytic FL, significantly outperforming both linear AFL and multi-round DeepAFL in accuracy across all benchmarks, demonstrating a highly efficient and scalable solution for federated vision.
\end{abstract}

\section{Introduction}

Federated Learning (FL) enables multiple clients or devices to collaboratively train a shared model without exposing their private data. 
Instead of centralizing data, clients perform local updates and periodically communicate model parameters to a server, which aggregates them into a global model~\cite{McMahan2017FedAvg}. 
While conceptually appealing, conventional FL frameworks require many communication rounds—often hundreds or thousands—for a model to converge. 
In practical deployments, clients can operate at different speeds, disconnect intermittently, or fail mid-training, creating stragglers and asynchronous updates. 
Such instability causes training to progress unevenly, and the global model may take days or weeks to reach convergence, severely limiting FL's real-world scalability.

Beyond communication inefficiency, a deeper issue lies in \textbf{statistical heterogeneity} across clients. 
In real FL systems, local data distributions often differ sharply—for instance, users capture different visual styles, hospitals record different patient populations, or sensors observe non-overlapping environments. 
This non-IID nature of the data means that each client’s gradient direction diverges from the global optimum, degrading performance and convergence stability. 
Existing methods attempt to address this through various regularizers, dynamic aggregation schemes, personalization strategies, architectural adaptations, and the use of pre-trained models for initialization and distillation~\cite{Li2020FedProx,Wang2020FedNova,Li2021MOON,Acar2021FedDyn,yang2023fedfed,pieri2023handling,lee2023fedl2p,bao2024provable}, but these methods still struggle under strong non-IID settings.

To overcome these limitations, recent work proposed \textit{Analytic Federated Learning} (AFL)~\cite{He2025AFL}, which formulates the FL problem in closed form. 
AFL leverages a pre-trained backbone to extract embeddings on each client, and trains a linear regression head analytically in only one communication round. 
Its analytic aggregation law guarantees invariance to both data partitioning and client count, enabling the global solution to remain identical to centralized training regardless of heterogeneity. 
As a result, AFL achieves higher accuracy than conventional iterative FL methods under highly non-IID conditions, while requiring only a single communication round instead of hundreds. 
Despite these appealing properties, AFL remains constrained by its linear model structure, which limits representational capacity and the ability to capture nonlinear feature interactions.

More recently, DeepAFL~\cite{deepafl2025} proposed a layer-wise analytic training scheme that extends AFL into deeper architectures. 
DeepAFL retains AFL’s invariance property, but trades communication efficiency for greater accuracy. 
Each analytic layer requires a separate aggregation round, increasing synchronization overhead and deviating from AFL’s single-pass analytic design. 
Consequently, DeepAFL achieves higher accuracy than AFL on non-IID data but at the cost of multiple communication rounds. %and larger transmission loads.

In parallel, recent one-shot federated learning methods have explored ensembling, synthetic data distillation, and pre-trained feature statistics to improve performance under severe heterogeneity~\cite{Allouah2024Revisiting,zhang2024fedsd2c,guan2025fedcgs}. 
However, these methods still generally lack the exact centralized-equivalence and heterogeneity-invariance guarantees provided by analytic formulations.

In this work, we propose \textbf{SAFLe} --- \textit{Sparse Analytic Federated Learning with nonlinear embeddings} --- a framework that retains AFL’s single-round analytic formulation while significantly enhancing model expressivity. SAFLe introduces a deterministic nonlinear transformation pipeline composed of three stages: \textit{feature bucketing}, \textit{shuffling and grouping}, and \textit{sparse embeddings}. We prove that this nonlinear transformation pipeline can be reformulated as an equivalent analytic regression problem, preserving AFL’s closed-form training and invariance properties.

This design allows SAFLe to scale model capacity by simply increasing the number of sparse embeddings, without altering the analytic formulation or introducing extra communication rounds. Empirically, SAFLe achieves higher accuracy than both AFL and DeepAFL across all datasets, including highly non-IID and large-client settings, while maintaining a single-round communication regime. %and lower memory footprint.

Our main contributions are summarized as follows:
\begin{itemize}
    \item \textbf{A Novel Non-Linear Analytic Framework (SAFLe):} We propose SAFLe, a framework that uses a sparse, multi-embedding architecture to learn non-linear feature interactions, dramatically increasing model expressivity over previous analytic methods.

    \item \textbf{Proof of Linear Equivalence:} We prove that this complex non-linear model is mathematically equivalent to a high-dimensional linear regression. This is the key theoretical insight that makes it analytically solvable.

    \item \textbf{Single-Round and Invariant Aggregation:} By leveraging this equivalence, SAFLe is the first non-linear framework to inherit AFL's two key properties: it converges in a single communication round and its solution is mathematically invariant to statistical data heterogeneity.

    \item \textbf{State-of-the-Art Analytic Performance:} SAFLe establishes a new state-of-the-art for analytic federated learning, significantly outperforming prior single-round (AFL) and multi-round (DeepAFL) methods.
\end{itemize}

\section{Background and Preliminaries}

\subsection{Federated Learning}
Conventional iterative FL methods, pioneered by FedAvg \cite{McMahan2017FedAvg}, train a global model through many communication rounds by alternating local client updates with server-side aggregation. While effective, this paradigm is highly vulnerable to statistical heterogeneity across clients. Numerous methods have been proposed to mitigate this issue, including regularization-based approaches such as FedProx \cite{Li2020FedProx} and FedDyn \cite{Acar2021FedDyn}, normalized aggregation in FedNova \cite{Wang2020FedNova}, contrastive learning in MOON \cite{Li2021MOON}, feature distillation in FedFed \cite{yang2023fedfed}, and personalization with pre-trained models in FedL2P \cite{lee2023fedl2p}. Recent work has also highlighted the importance of architectural design in handling heterogeneity for visual FL \cite{pieri2023handling}. Despite these advances, these methods remain fundamentally gradient-based and iterative, incurring substantial communication overhead.

To reduce communication, One-Shot Federated Learning (OFL) aims to complete training in a single round \cite{Allouah2024Revisiting}. Existing OFL methods typically rely on server-side ensembling \cite{Allouah2024Revisiting}, knowledge distillation \cite{Dai2024CoBoosting,Wu2022FedKD}, or more recent synthetic distillation and global feature statistics, as in FedSD2C \cite{zhang2024fedsd2c} and FedCGS \cite{guan2025fedcgs}. However, OFL methods generally still suffer a noticeable accuracy gap relative to iterative FL, especially under strong heterogeneity \cite{Allouah2024Revisiting}.

\subsection{Analytic Federated Learning (AFL)}
Analytic Federated Learning (AFL) \cite{He2025AFL} was proposed to solve both problems. It is a gradient-free, single-round framework that uses a pre-trained frozen backbone. On each client $k$, a linear classification head is trained by solving a simple least-squares (LS) problem:
\begin{equation}
 \min_{W_k} \mathcal{L}(W_k) = ||Y_k - X_k W_k||_F^2 
\end{equation}
where $X_k$ is the matrix of backbone embeddings and $Y_k$ is the one-hot label matrix.

AFL's core contribution is the \textbf{Absolute Aggregation (AA) law}. This law provides an analytic formula for the server to perfectly reconstruct the centralized global model $W$ (the model that would have been trained on all data at once) from the local models ($W_u$, $W_v$, etc.) in a single round. For two clients $u$ and $v$, the aggregation is:
\begin{equation}
 W = \mathcal{W}_u W_u + \mathcal{W}_v W_v 
\end{equation}
The weighting matrices, $\mathcal{W}_u$ and $\mathcal{W}_v$, are not simple averages; they are computed analytically from the clients' data covariance matrices ($C_u = X_u^T X_u$, $C_v = X_v^T X_v$). Because this formula is an exact, closed-form solution, the final model $W$ is mathematically invariant to how the data is partitioned, making it robust to any degree of non-IID data or client count. To handle rank-deficient data, a \textbf{Regularization Intermediary (RI)} process is used to add and then analytically remove regularization, preserving the exact solution.

AFL is thus extremely fast and robust, but its primary limitation is its constrained representational power, as it can only train a single \textit{linear} layer.

\subsection{Deep Analytic Federated Learning (DeepAFL)}
DeepAFL \cite{deepafl2025} was introduced to address AFL's linearity constraint by building a deeper, non-linear analytic model inspired by ResNet, that achieves non-linearity by using layers with random weights.

However, DeepAFL sacrifices AFL's single-shot efficiency for this added expressivity. It uses a sequential, layer-wise training protocol that requires multiple communication rounds per layer. For each layer $t$ in the model, the server must: $1)$ Perform a full aggregation round (like AFL) to solve for the layer's classifier, $W_t$. $2)$ Perform a \textit{second}, separate aggregation round to collect components needed to solve for a new feature transformation, $\Omega_{t+1}$, which uses random projections to build the input for the next layer.

This multi-round process must be repeated for every layer added to the model. While DeepAFL achieves higher accuracy than AFL, it re-introduces significant synchronization overhead. This creates a clear trade-off: AFL is single-round but linear, while DeepAFL is non-linear but multi-round.

\subsection{Federated WiSARD Models}
A related direction \cite{eff_kno} explores federated aggregation with Weightless Neural Networks (WNNs) \cite{wnn_intro_esann,wisard}, using counter-based RAM memories to combine local knowledge across clients in a way that is invariant to how the data is partitioned. However, unlike analytic federated learning methods, these approaches do not derive the global model from a closed-form analytical solution; instead, they rely on heuristic-based learning and approximate memory mechanisms, which leads to lower accuracy. As a result, although they retain heterogeneity-invariant aggregation through counter-based RAM structures, they lack the precision and performance of analytically derived federated solutions.

\section{Methodology}

\begin{figure*}[t]
\centering
\includegraphics[width=\linewidth]{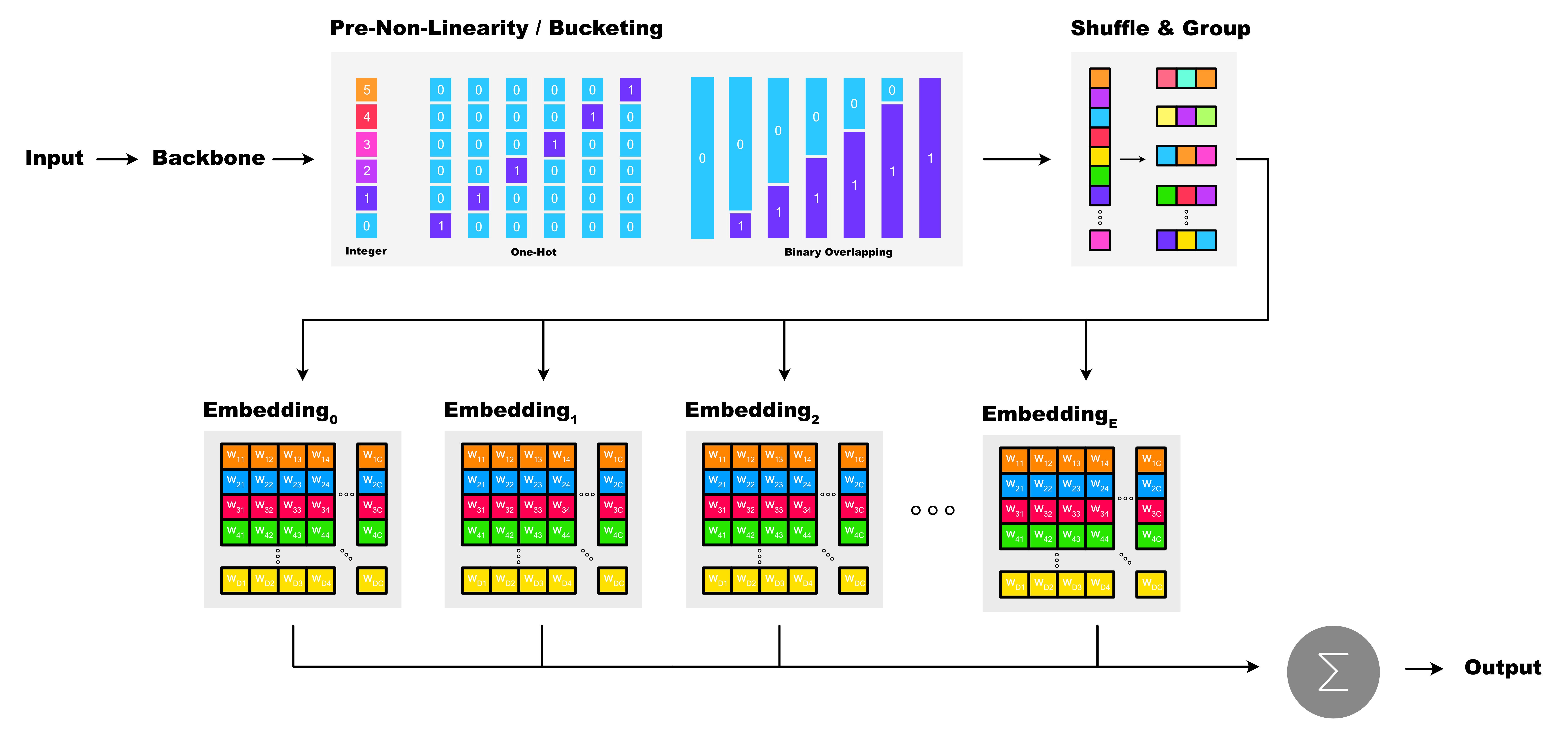}
\caption{Illustration of the proposed SAFLe model. Input images are first processed by a pre-trained backbone network to extract features. A Pre-Non-Linearity by Bucketing transformation is then applied to each feature using one of three methods: Integer, One-Hot, or Binary Overlapping bucketing. The resulting bucketed features are shuffled and divided into $E$ groups, each of which is passed through a corresponding learnable embedding. The outputs of all embeddings are summed to produce the final model output. The embedding parameters admit an analytical solution and can be optimized through a federated learning algorithm that remains invariant to both data distribution and the number of participating clients.}
\label{fig:fedbase}
\end{figure*}

\subsection{Intuition}

The primary limitation of AFL is its linear model, which restricts representational capacity. The key challenge is to introduce non-linearity while preserving the single-round, analytic solution. DeepAFL, the first attempt to solve this, derives its non-linearity from random projections. This stochastic approach is inefficient; it relies on the chance that stacking enough random matrices and activations will eventually approximate the complex, non-linear landscape of the data. Consequently, it requires many layers to achieve high accuracy, which in turn re-introduces the multi-round communication burden that AFL was designed to eliminate.

We propose a different, more structured approach. Instead of relying on random weights, our core idea is to deterministically partition the continuous feature space using bucketing. This creates a set of discrete "regions" in the data landscape. We can then use embedding layers to directly learn the optimal output (e.g., logits) for inputs that fall into specific combinations of these regions. This directly models the non-linear function. A naive implementation, however, would be an intractable lookup table that would severely overfit. We solve this by introducing a sparse, multi-embedding architecture. The feature-buckets are shuffled and partitioned into many small, independent groups ($E$), and each embedding layer learns from only one of these "subviews." The final prediction is the sum of all these "expert" embeddings. See Figure \ref{fig:fedbase}. This forces generalization, as the model learns to make predictions from diverse, partial feature-combinations. This method is fundamentally more structured than DeepAFL's, and as we will prove, it remains fully analytic and solvable in a single round.

This architecture is conceptually similar to a Mixture of Experts (MoE) model \cite{jacobs1991adaptive}, as it employs a set of "experts" (our $E$ embedding layers) to learn specialized functions over different parts of the input space. However, SAFLe differs in two crucial ways: its routing mechanism and its activation pattern. First, instead of a \textit{learned} gating network that dynamically routes an input, SAFLe uses our \textit{deterministic} bucketing-and-shuffling pipeline as a fixed pre-router. Second, rather than the sparse activation of a standard MoE, SAFLe uses dense activation: \textit{all} $E$ experts are activated for \textit{every} input.

This design is the key to both its expressivity and its analytic nature. Each expert is forced to learn a specialized, non-linear function based on only a small, fixed "subview" of the total input features. The model's full non-linear capacity comes from summing the contributions of all these parallel "subview experts." Crucially, this fixed, deterministic routing—unlike a standard MoE's learned gate—allows us to reformulate the entire non-linear architecture as an equivalent high-dimensional linear regression. This reformulation is precisely what makes our model compatible with AFL's Absolute Aggregation (AA) law, enabling a single-round, gradient-free solution.

\subsection{SAFLe}
\label{sec:arch}

Our objective is to break AFL's linearity barrier while retaining its single-round, gradient-free properties. We replace AFL's simple linear regressor with a deterministic, non-linear transformation pipeline, as illustrated in Figure~\ref{fig:fedbase}.

Let the output of the pre-trained backbone be a feature vector $x \in \mathbb{R}^{d_b}$, where $d_b$ is the dimension of the backbone's feature output. Our non-linear head, $f_{NL}(x)$, transforms this vector into the final class logits $\hat{y} \in \mathbb{R}^{C}$. This transformation consists of three stages:

\textbf{1. Pre-Non-Linearity Bucketing:}
First, we apply $L$ different bucketing functions $B_l(\cdot)$ to each feature $x_i$ in the backbone output. Each function $B_l: \mathbb{R} \to \{0, \dots, k-1\}$ quantizes the continuous feature into one of $k$ discrete bins. This transforms the original vector $x \in \mathbb{R}^{d_b}$ into a new quantized integer vector $b \in \mathbb{Z}^{d_q}$, where the new dimension $d_q = d_b \times L$.
\begin{equation}
 b = [B_{1}(x_1), \dots, B_{L}(x_1), \dots, B_{1}(x_{d_b}), \dots, B_{L}(x_{d_b})] 
\end{equation}
The choice of bucketing strategy is discussed in our ablations (Section~\ref{sec:experiments}).

\textbf{2. Shuffling and Grouping:}
To break feature locality and create diverse "subviews" for our experts, the integer vector $b$ is shuffled using a fixed, deterministic permutation $P$. This shuffled vector $b' = P(b)$ is then partitioned into $E$ groups (our "experts"). Each group $j$ contains $G$ integer indices, $g_j = [b'_{j,1}, \dots, b'_{j,G}]$, such that $E \times G = d_q$.

\textbf{3. Sparse Embedding and Summation:}
For each of the $E$ groups, we compute a single composite index $idx_j$ by treating the $G$ integers as digits in a base-$k$ number system:
\begin{equation}
 idx_j = \sum_{i=1}^{G} g_{j,i} \cdot k^{i-1} 
\end{equation}
This composite index has a maximum value of $V = k^G$, which defines the "vocabulary size" (i.e., number of rows) for each of our $E$ embedding matrices. Each embedding matrix $W_j \in \mathbb{R}^{V \times C}$ thus maps its specific group of $G$ bucketed features to a $C$-dimensional output vector.

The final output logit $\hat{y}$ is the dense activation and sum of the outputs from all $E$ embedding-experts:
\begin{equation}
 \hat{y} = f_{NL}(x) = \sum_{j=1}^{E} W_j[idx_j, :] 
\end{equation}
As we show in our ablations, using many small, independent embeddings (high $E$, small $G$ and $V$) is critical for generalization.

\subsection{Analytic Formulation and Aggregation}
We now prove that the proposed non-linear architecture $f_{NL}$ can be solved with a closed-form analytic solution. We first formalize the linear equivalence.

\begin{lemma}
\textbf{(SAFLe Linear Equivalence).} The non-linear model $f_{NL}(x)$, which learns feature interactions via grouped embedding lookups, is equivalent to a linear regressor $W_{global}$ in a high-dimensional sparse feature space $\Phi(x)$.
\end{lemma}

\begin{proof}
The non-linearity stems from the grouped indexing (Eq. 2), where a group of $G$ features $\{g_{j,1}, \dots, g_{j,G}\}$ is combined to form a single index $idx_j$. The model's output (Eq. 3) is the sum of $E$ such lookups.

To solve for the weights $W_j$ analytically, we reformulate this lookup. An embedding lookup $W_j[idx_j, :]$ is a linear row-selection operation, which can be expressed using a one-hot vector $\phi_j(x) \in \{0, 1\}^V$ where the only non-zero entry is at position $idx_j$:
\begin{equation}
 W_j[idx_j, :] = \phi_j(x)^T W_j 
\end{equation}
The total output $\hat{y}$ is the sum of all $E$ lookups:
\begin{equation}
 \hat{y} = \sum_{j=1}^{E} \phi_j(x)^T W_j 
\end{equation}
We now define two global structures. First, a single, high-dimensional sparse feature vector $\Phi(x)$ by horizontally concatenating all $E$ one-hot vectors:
\begin{equation}
 \Phi(x) = [\phi_1(x)^T | \phi_2(x)^T | \dots | \phi_E(x)^T]^T \in \mathbb{R}^{D_e} 
\end{equation}
where the total embedding dimension is $D_e = E \times V$.
Second, a global weight matrix $W_{global}$ by vertically stacking all $E$ embedding matrices:
\begin{equation}
 W_{global} = \begin{pmatrix} W_1 \\ W_2 \\ \vdots \\ W_E \end{pmatrix} \in \mathbb{R}^{D_e \times C} 
\end{equation}
With these structures, the output of our non-linear model $f_{NL}(x)$ is perfectly represented as a standard linear model:
\begin{equation}
 \hat{y} = \Phi(x)^T W_{global} 
\end{equation}
For an entire batch of data $X_k$ on client $k$, we can construct the full linearized feature matrix $\Phi_k$ (of size $N_k \times D_e$). The global objective over $K$ clients is to find $W_{global}$ by minimizing the least-squares loss:
\begin{equation}
 \mathcal{L}(W_{global}) = \sum_{k=1}^K ||Y_k - \Phi_k W_{global}||_F^2 = ||Y - \Phi W_{global}||_F^2 
\end{equation}
This objective is mathematically identical to the one solved by AFL, simply replacing their feature matrix $X_k$ with our high-dimensional sparse feature matrix $\Phi_k$.
\end{proof}

This equivalence allows us to apply AFL's federated aggregation framework. The optimal global model $W_{global}$ is the solution to the normal equations:
\begin{equation}
 W_{global} = (\Phi^T \Phi)^\dagger (\Phi^T Y) 
\end{equation}
where the global components are simple sums of the local components:
\begin{equation}
 \Phi^T \Phi = \sum_{k=1}^K \Phi_k^T \Phi_k \quad \text{and} \quad \Phi^T Y = \sum_{k=1}^K \Phi_k^T Y_k 
\end{equation}
This structure is identical to AFL and can be solved in a single round using the Regularization Intermediary (RI) process.

\begin{theorem}
\textbf{(Federated Analytic Solution).} By extending the RI-AA Law from \cite{He2025AFL}, the global model $W_{global}$ can be solved analytically in a single communication round.
\end{theorem}

\begin{proof}
Following the RI process \cite{He2025AFL}, each client $k$ computes and transmits two matrices based on its local high-dimensional features $\Phi_k$:
\begin{align}
 C_k^r &= \Phi_k^T \Phi_k + \gamma I \in \mathbb{R}^{D_e \times D_e} \\
 M_k &= \Phi_k^T Y_k \in \mathbb{R}^{D_e \times C} 
\end{align}
The server aggregates these components in a single step:
\begin{equation}
 C_{agg}^r = \sum_{k=1}^{K} C_k^r \quad \text{and} \quad M_{agg} = \sum_{k=1}^{K} M_k 
\end{equation}
The server first computes the aggregated regularized solution:
\begin{equation}
 W_{global}^r = (C_{agg}^r)^{-1} M_{agg} 
\end{equation}
Finally, the server analytically removes the regularization to recover the exact, unregularized global solution $W_{global}$ using the recovery formula from AFL:
\begin{equation}
 W_{global} = (C_{agg}^r - K\gamma I)^\dagger M_{agg} 
\end{equation}
Thus, we have derived a closed-form, single-round analytic solution for our non-linear model. The computation is tractable as $D_e = E \times V$ is a controllable hyperparameter.
\end{proof}

\section{Experiments}
\label{sec:experiments}

\subsection{Experimental Setup}

\textbf{Datasets \& Settings.}
We conduct our experiments on three prominent federated learning benchmark datasets: CIFAR-10 \cite{Krizhevsky2009Learning}, CIFAR-100 \cite{Krizhevsky2009Learning}, and Tiny-ImageNet \cite{Le2015Tiny}. To simulate statistical heterogeneity, we follow the standard setup used in AFL \cite{He2025AFL} and DeepAFL \cite{deepafl2025} and employ two common non-IID partitioning strategies \cite{Lin2020Ensemble}: Non-IID-1 (LDA), where data is partitioned using a Latent Dirichlet Allocation (LDA) process controlled by the parameter $\alpha$, and Non-IID-2 (Sharding), where data is sorted by label and divided into shards, with each client receiving $s$ shards. In both settings, smaller values of $\alpha$ or $s$ indicate a higher degree of data heterogeneity. For CIFAR-10, we test with $\alpha \in \{0.1, 0.05\}$ and $s \in \{2, 4\}$. For the more complex CIFAR-100 and Tiny-ImageNet datasets, we use $\alpha \in \{0.1, 0.01\}$ and $s \in \{5, 10\}$. Unless otherwise specified, all experiments use a total of $K=100$ clients.

\textbf{Baselines \& Metrics.}
We compare our proposed SAFLe against two categories of federated learning methods. The first category, Iterative Baselines, includes five widely-used gradient-based methods: FedAvg \cite{McMahan2017FedAvg}, FedProx \cite{Li2020FedProx}, MOON \cite{Li2021MOON}, FedDyn \cite{Acar2021FedDyn}, and FedNTD \cite{Lee2022FedNTD}. The second, Analytic Baselines, includes the two most relevant analytic methods: AFL \cite{He2025AFL}, the single-round linear analytic model, and DeepAFL \cite{deepafl2025}, the state-of-the-art multi-round analytic model. For a fair and direct comparison, all methods (both iterative and analytic) are initialized with the same ResNet-18 backbone pre-trained on ImageNet. The results for all baseline methods are directly obtained from the established benchmark provided in \cite{He2025AFL, deepafl2025}. We evaluate all approaches on Top-1 Accuracy (\%) to measure model performance, Communication Rounds to measure synchronization overhead, and Communication Cost (MB) to measure total data transmission per client.

\subsection{Experimental Results}

\textbf{Accuracy Comparisons}
The main accuracy results, presented in Table \ref{tab:reconstructed_table_condensed}, show that SAFLe consistently and significantly outperforms all iterative and analytic baselines across all datasets and non-IID settings. On CIFAR-10, SAFLe achieves 90.73\% accuracy, surpassing the next-best analytic method, DeepAFL \cite{deepafl2025} (86.43\%), by over 4.3\% and the original AFL \cite{He2025AFL} (80.75\%) by a 10\% margin. On CIFAR-100, SAFLe (70.61\%) again leads DeepAFL (66.98\%) and AFL (58.56\%), and on Tiny-ImageNet, our method (64.58\%) continues to outperform DeepAFL (62.35\%) and AFL (54.67\%). This demonstrates the superior representational power of our proposed non-linear analytic architecture, which achieves higher accuracy than even the multi-round DeepAFL while maintaining single-round efficiency.

\begin{table*}[t]
\footnotesize
\centering
\caption{The top-1 accuracy (\%) of compared methods under two non-IID settings. Settings controlled by $\alpha$ and $s$ are NIID-1 and NIID-2, respectively. The data is reported as average and standard deviation after multiple runs. Results in bold are the best within the compared methods in the same setting.}
\label{tab:reconstructed_table_condensed}
\begin{tabular}{llccccccccc}
\toprule
\textbf{Dataset} & \textbf{Setting} & \textbf{FedAvg} & \textbf{FedProx} & \textbf{MOON} & \textbf{FedDyn} & \textbf{FedNTD} & \textbf{AFL} & \textbf{DeepAFL} & \textbf{SAFLe} \\
\midrule
\multirow{4}{*}{CIFAR-10}
& $\alpha=0.1$ & $64.02_{\pm0.18}$ & $64.07_{\pm0.08}$ & $63.84_{\pm0.03}$ & $64.77_{\pm0.11}$ & $64.64_{\pm0.02}$ & \cellcolor{white!10}$80.75_{\pm0.00}$ & \cellcolor{white!10}$86.43_{\pm0.07}$ & \cellcolor{white!10}$\mathbf{90.73}_{\pm\mathbf{0.07}}$ \\
& $\alpha=0.05$ & $60.52_{\pm0.39}$ & $60.39_{\pm0.09}$ & $60.28_{\pm0.17}$ & $60.35_{\pm0.54}$ & $61.16_{\pm0.33}$ & \cellcolor{white!10}$80.75_{\pm0.00}$ & \cellcolor{white!10}$86.43_{\pm0.07}$ & \cellcolor{white!10}$\mathbf{90.73}_{\pm\mathbf{0.07}}$ \\
& $s=4$ & $68.47_{\pm0.13}$ & $68.46_{\pm0.08}$ & $68.47_{\pm0.15}$ & $73.50_{\pm0.11}$ & $70.24_{\pm0.11}$ & \cellcolor{white!10}$80.75_{\pm0.00}$ & \cellcolor{white!10}$86.43_{\pm0.07}$ & \cellcolor{white!10}$\mathbf{90.73}_{\pm\mathbf{0.07}}$ \\
& $s=2$ & $57.81_{\pm0.03}$ & $57.61_{\pm0.12}$ & $57.72_{\pm0.15}$ & $64.07_{\pm0.09}$ & $58.77_{\pm0.18}$ & \cellcolor{white!10}$80.75_{\pm0.00}$ & \cellcolor{white!10}$86.43_{\pm0.07}$ & \cellcolor{white!10}$\mathbf{90.73}_{\pm\mathbf{0.07}}$ \\
\midrule
\multirow{4}{*}{CIFAR-100} 
& $\alpha=0.1$ & $56.62_{\pm0.12}$ & $56.45_{\pm0.22}$ & $56.58_{\pm0.02}$ & $57.55_{\pm0.08}$ & $56.60_{\pm0.14}$ & \cellcolor{white!10}$58.56_{\pm0.00}$ & \cellcolor{white!10}$66.98_{\pm0.04}$ & \cellcolor{white!10}$\mathbf{70.61}_{\pm\mathbf{0.14}}$ \\
& $\alpha=0.01$ & $32.99_{\pm0.20}$ & $33.37_{\pm0.09}$ & $33.34_{\pm0.11}$ & $36.12_{\pm0.08}$ & $32.59_{\pm0.21}$ & \cellcolor{white!10}$58.56_{\pm0.00}$ & \cellcolor{white!10}$66.98_{\pm0.04}$ & \cellcolor{white!10}$\mathbf{70.61}_{\pm\mathbf{0.14}}$ \\
& $s=10$ & $55.76_{\pm0.13}$ & $55.80_{\pm0.16}$ & $55.70_{\pm0.25}$ & $61.09_{\pm0.09}$ & $54.69_{\pm0.15}$ & \cellcolor{white!10}$58.56_{\pm0.00}$ & \cellcolor{white!10}$66.98_{\pm0.04}$ & \cellcolor{white!10}$\mathbf{70.61}_{\pm\mathbf{0.14}}$ \\
& $s=5$ & $48.33_{\pm0.15}$ & $48.29_{\pm0.14}$ & $48.34_{\pm0.19}$ & $59.34_{\pm0.11}$ & $47.00_{\pm0.19}$ & \cellcolor{white!10}$58.56_{\pm0.00}$ & \cellcolor{white!10}$66.98_{\pm0.04}$ & \cellcolor{white!10}$\mathbf{70.61}_{\pm\mathbf{0.14}}$ \\
\midrule
\multirow{4}{*}{\shortstack{Tiny \\ ImageNet}} 
& $\alpha=0.1$ & $46.04_{\pm0.27}$ & $46.47_{\pm0.23}$ & $46.21_{\pm0.14}$ & $47.72_{\pm0.22}$ & $46.17_{\pm0.16}$ & \cellcolor{white!10}$54.67_{\pm0.00}$ & \cellcolor{white!10}$62.35_{\pm0.01}$ & \cellcolor{white!10}$\mathbf{64.58}_{\pm\mathbf{0.23}}$ \\
& $\alpha=0.01$ & $32.63_{\pm0.19}$ & $32.26_{\pm0.14}$ & $32.38_{\pm0.20}$ & $35.19_{\pm0.06}$ & $31.86_{\pm0.44}$ & \cellcolor{white!10}$54.67_{\pm0.00}$ & \cellcolor{white!10}$62.35_{\pm0.01}$ & \cellcolor{white!10}$\mathbf{64.58}_{\pm\mathbf{0.23}}$ \\
& $s=10$ & $39.06_{\pm0.26}$ & $38.97_{\pm0.23}$ & $38.79_{\pm0.14}$ & $41.36_{\pm0.06}$ & $37.55_{\pm0.09}$ & \cellcolor{white!10}$54.67_{\pm0.00}$ & \cellcolor{white!10}$62.35_{\pm0.01}$ & \cellcolor{white!10}$\mathbf{64.58}_{\pm\mathbf{0.23}}$ \\
& $s=5$ & $29.66_{\pm0.19}$ & $29.17_{\pm0.16}$ & $29.24_{\pm0.30}$ & $35.18_{\pm0.18}$ & $29.01_{\pm0.14}$ & \cellcolor{white!10}$54.67_{\pm0.00}$ & \cellcolor{white!10}$62.35_{\pm0.01}$ & \cellcolor{white!10}$\mathbf{64.58}_{\pm\mathbf{0.23}}$ \\
\bottomrule
\end{tabular}
\end{table*}

\textbf{Invariance Analysis}
Table \ref{tab:reconstructed_table_condensed} empirically shows the core theoretical benefit of our analytic formulation: invariance to data heterogeneity. While all iterative gradient-based methods suffer severe performance degradation as the data becomes more non-IID (e.g., $\alpha=0.01$ or $s=5$), our method does not. For instance, on CIFAR-100, FedAvg's accuracy plummets from 56.62\% ($\alpha=0.1$) to 32.99\% ($\alpha=0.01$), a drop of over 23.6\%. In stark contrast, all analytic methods are completely unaffected by the data distribution. Our proposed SAFLe achieves an identical 70.61\% accuracy across all four non-IID settings on CIFAR-100, just as AFL (58.56\%) and DeepAFL (66.98\%) maintain their respective scores. This invariance also extends to the number of clients, as shown in Figure \ref{fig:client_number_comparison}. While FedAvg's performance clearly declines as $K$ increases from 100 to 1000, SAFLe's accuracy remains constant, matching the invariance of AFL. This experimentally showcases SAFLe’s successful extension of the mathematical invariance of AFL to a more expressive non-linear model.

\textbf{Communication Rounds Efficiency}
DeepAFL \cite{deepafl2025} scales AFL by adding \textit{depth} (i.e., more layers), but at the direct trade-off of more communication rounds. As seen in Figures \ref{fig:rounds} and \ref{fig:rounds_all}, DeepAFL must introduce two new communication rounds for every additional layer, requiring 41 rounds for a $T=20$ model. Our method, SAFLe, scales by adding \textit{width} (i.e., more experts $E$), allowing us to scale the model's expressivity while remaining a \textbf{single-shot} solution. SAFLe's expressivity is decoupled from its number of communication rounds, which is always one round.

% \begin{figure}[t]
% \centering
%     % Subfigure (a) for CIFAR-100
%     \begin{subfigure}[b]{0.48\columnwidth}
%         \includegraphics[width=\linewidth]{figures/pareto_frontier_acc_rounds_c100.png}
%         \caption{CIFAR-100}
%         \label{fig:cifar100_acc_rounds}
%     \end{subfigure}
%     \hfill % Adds horizontal space between the figures
%     % Subfigure (b) for Tiny-ImageNet
%     \begin{subfigure}[b]{0.48\columnwidth}
%         \includegraphics[width=\linewidth]{figures/pareto_frontier_acc_rounds_tinyimg.png}
%         \caption{Tiny-ImageNet}
%         \label{fig:tinyimagenet_acc_rounds}
%     \end{subfigure}
    
%     % Main figure caption
%     \caption{Caption}
%     \label{fig:pareto_frontier_acc_rounds}
% \end{figure}

\begin{figure}[t]
\centering
    % Subfigure (a) for CIFAR-100
    \begin{subfigure}[b]{0.48\columnwidth}
        \includegraphics[width=\linewidth]{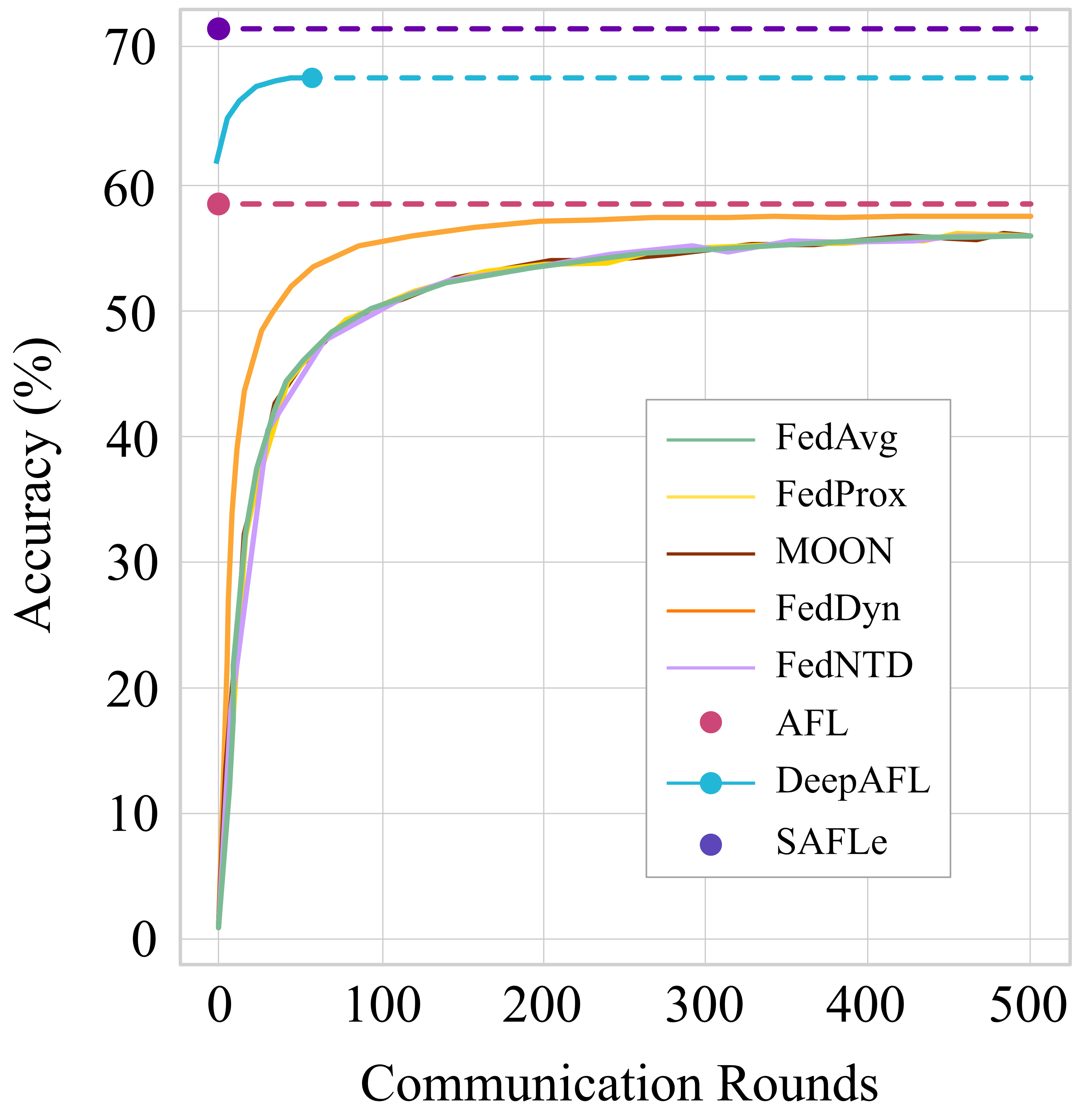}
        \caption{CIFAR-100}
        \label{fig:rounds_c100}
    \end{subfigure}
    \hfill % Adds horizontal space between the figures
    % Subfigure (b) for Tiny-ImageNet
    \begin{subfigure}[b]{0.48\columnwidth}
        \includegraphics[width=\linewidth]{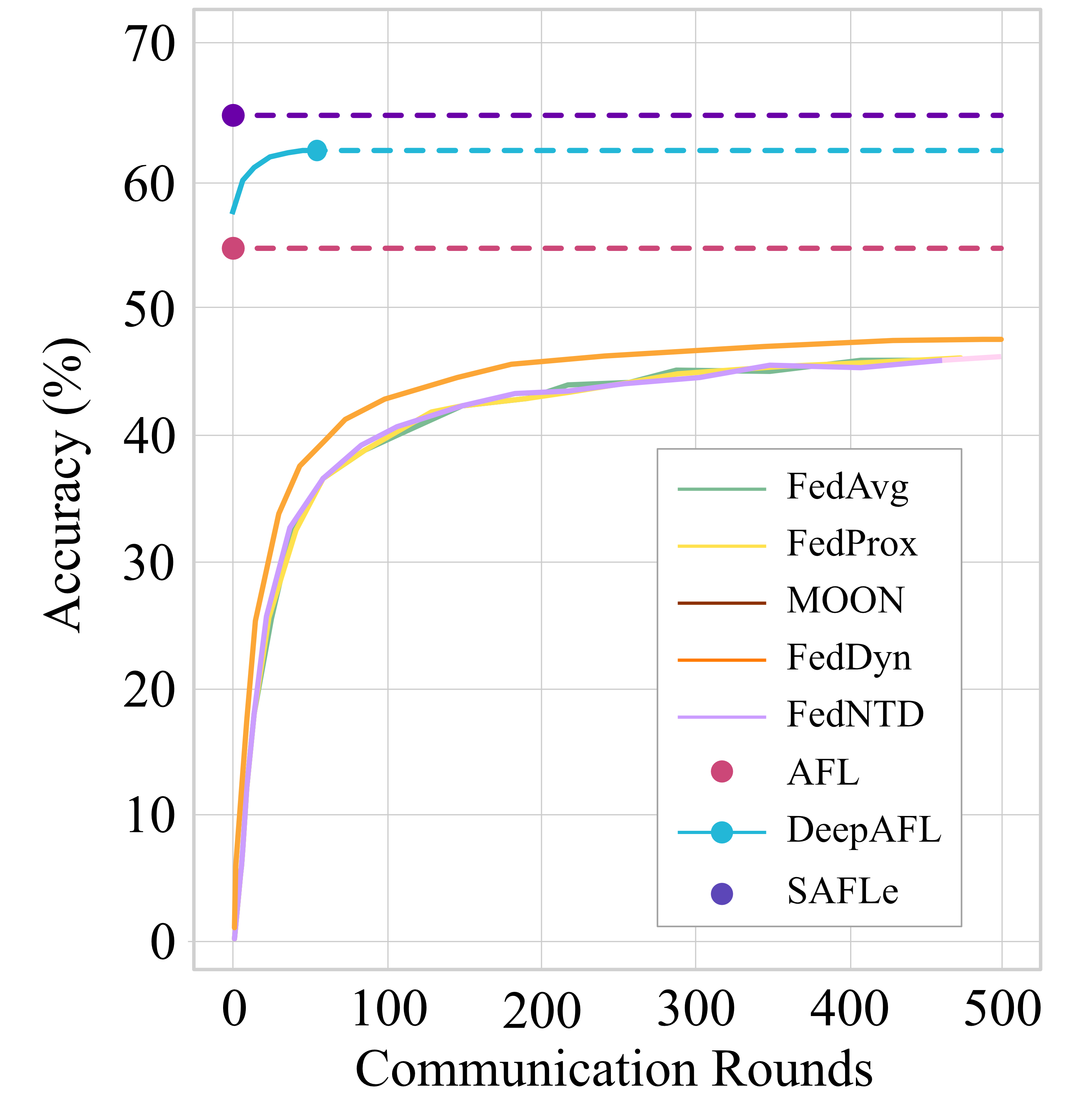}
        \caption{Tiny-ImageNet}
        \label{fig:rounds_tiny}
    \end{subfigure}
    
    % Main figure caption
    \caption{Accuracy vs. Communication Rounds on (a) CIFAR-100 and (b) Tiny-ImageNet. Analytic methods SAFLe and AFL achieve final accuracy in a single round.}
    \label{fig:rounds_all}
\end{figure}

\textbf{Communication Cost Efficiency}
While SAFLe operates in a single round, this design choice increases the size of the payload for that single transmission. However, as shown in Figure \ref{fig:pareto_frontier_acc_comm}, when comparing the total communication cost (MB) required to achieve a target accuracy, SAFLe is significantly more efficient than DeepAFL. For example, to reach $\approx 67\%$ accuracy on CIFAR-100, DeepAFL requires 146MB of total communication. Our method achieves this same accuracy using only 70MB, a reduction of over 50\%. This efficiency gain is even more pronounced on Tiny-ImageNet: DeepAFL needs 170MB to achieve 62.3\% accuracy, while SAFLe requires only 60MB, an almost 3x reduction. This low communication cost is possible due to the sparse nature of the embeddings, as discussed in the methodology, and experimentally shown in the next paragraph.

\begin{figure}[t]
\centering
    % Subfigure (a) for CIFAR-100
    \begin{subfigure}[b]{0.48\columnwidth}
        \includegraphics[width=\linewidth]{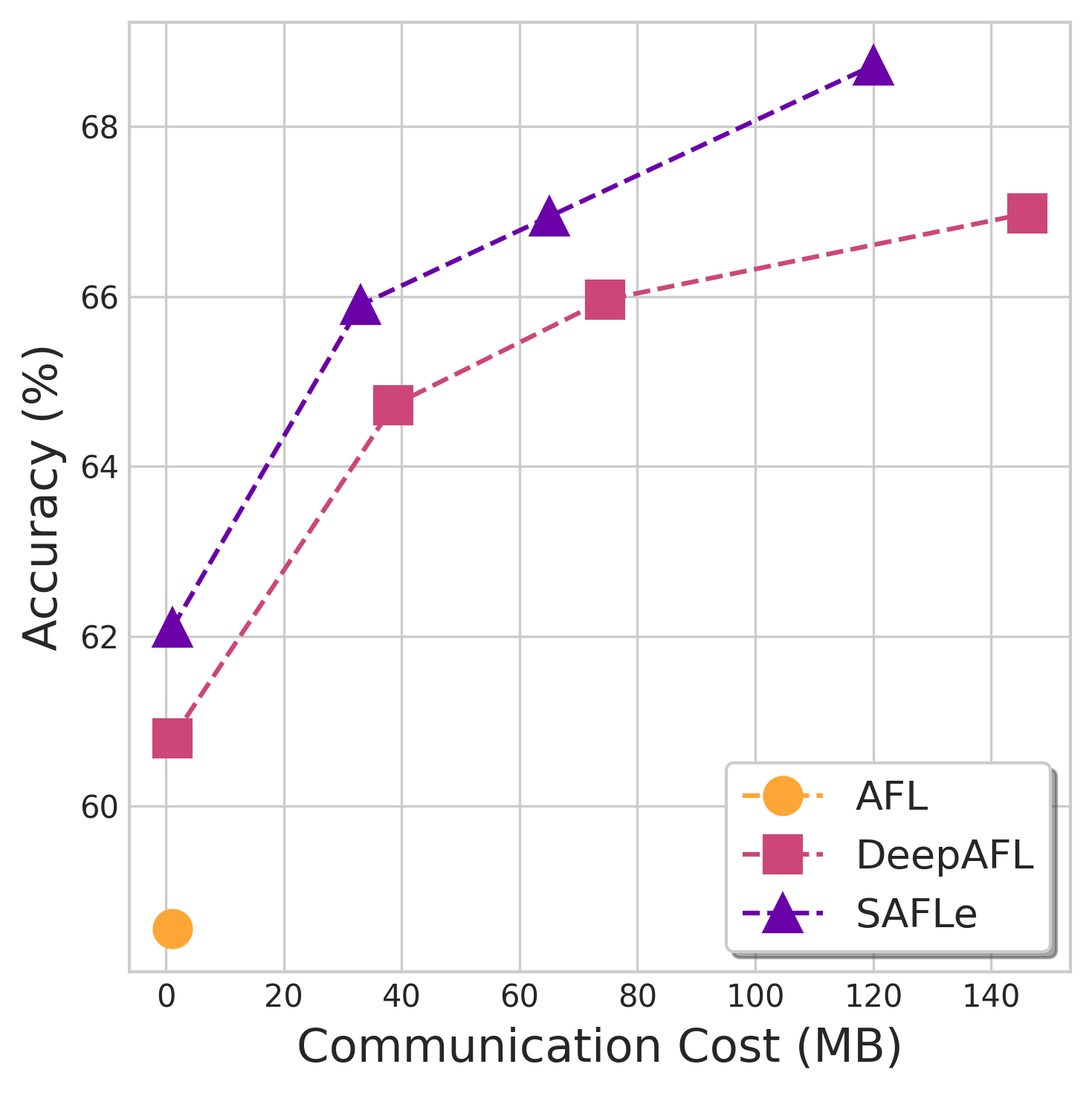}
        \caption{CIFAR-100}
        \label{fig:acc_comm_c100}
    \end{subfigure}
    \hfill % Adds horizontal space between the figures
    % Subfigure (b) for Tiny-ImageNet
    \begin{subfigure}[b]{0.48\columnwidth}
        \includegraphics[width=\linewidth]{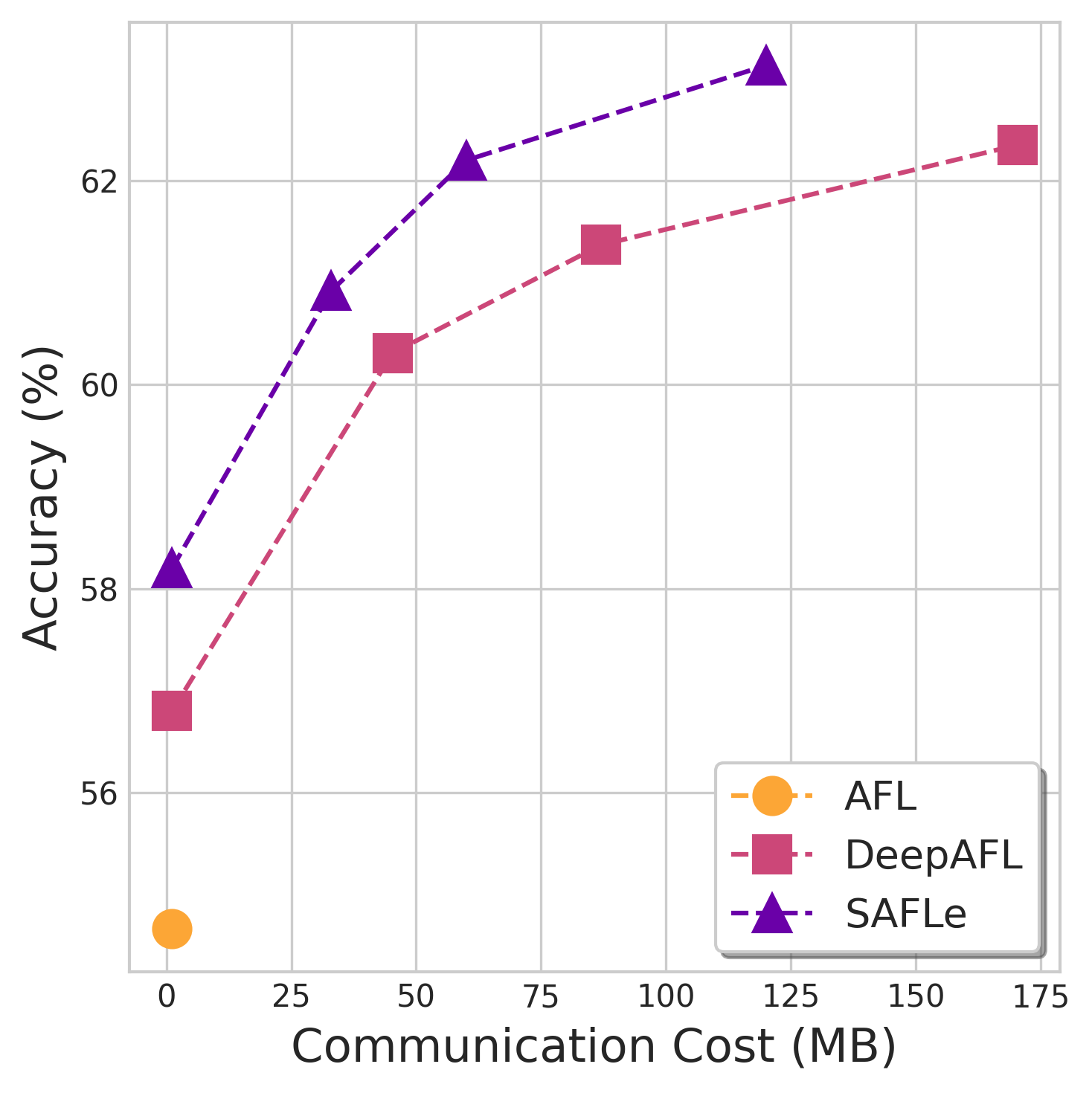}
        \caption{Tiny-ImageNet}
        \label{fig:acc_comm_tiny}
    \end{subfigure}
    % Main figure caption
    \caption{Accuracy vs. Total Communication Cost (MB) per client on (a) CIFAR-100 and (b) Tiny-ImageNet. Our single-round method, SAFLe, achieves a given accuracy (e.g., $\approx 62\%$ on Tiny-ImageNet) for a fraction of the total communication cost required by the multi-round DeepAFL.}
    \label{fig:pareto_frontier_acc_comm}
\end{figure}

% \textbf{Model Size Comparison.}
% We note that the additional parameters introduced by our SAFLe architecture—the sparse embedding layers—are negligible in size compared to the backbone. In all configurations, the total size of these embeddings is less than 0.5MB, which is insignificant compared to the $\approx 44$MB ResNet-18. This ensures a fair comparison, as it demonstrates that our method's significant accuracy improvements are not attributable to an unfair advantage in total model parameters. While SAFLe's non-linear design *does* scale model capacity, which is the source of its improved expressivity, it does so with a fractional parameter increase so small that it does not to provide any meaningful unfair advantage to the standard iterative, gradient-based baselines. As a reference, this additional size is equivalent to adding approximately 146 neurons to the final hidden layer head of the other SGD baseline methods, which produces no meaningful change in accuracy.

\begin{figure}[t]
\centering
    % Subfigure (a) for CIFAR-100
    \begin{subfigure}[b]{0.48\columnwidth}
        \includegraphics[width=\linewidth]{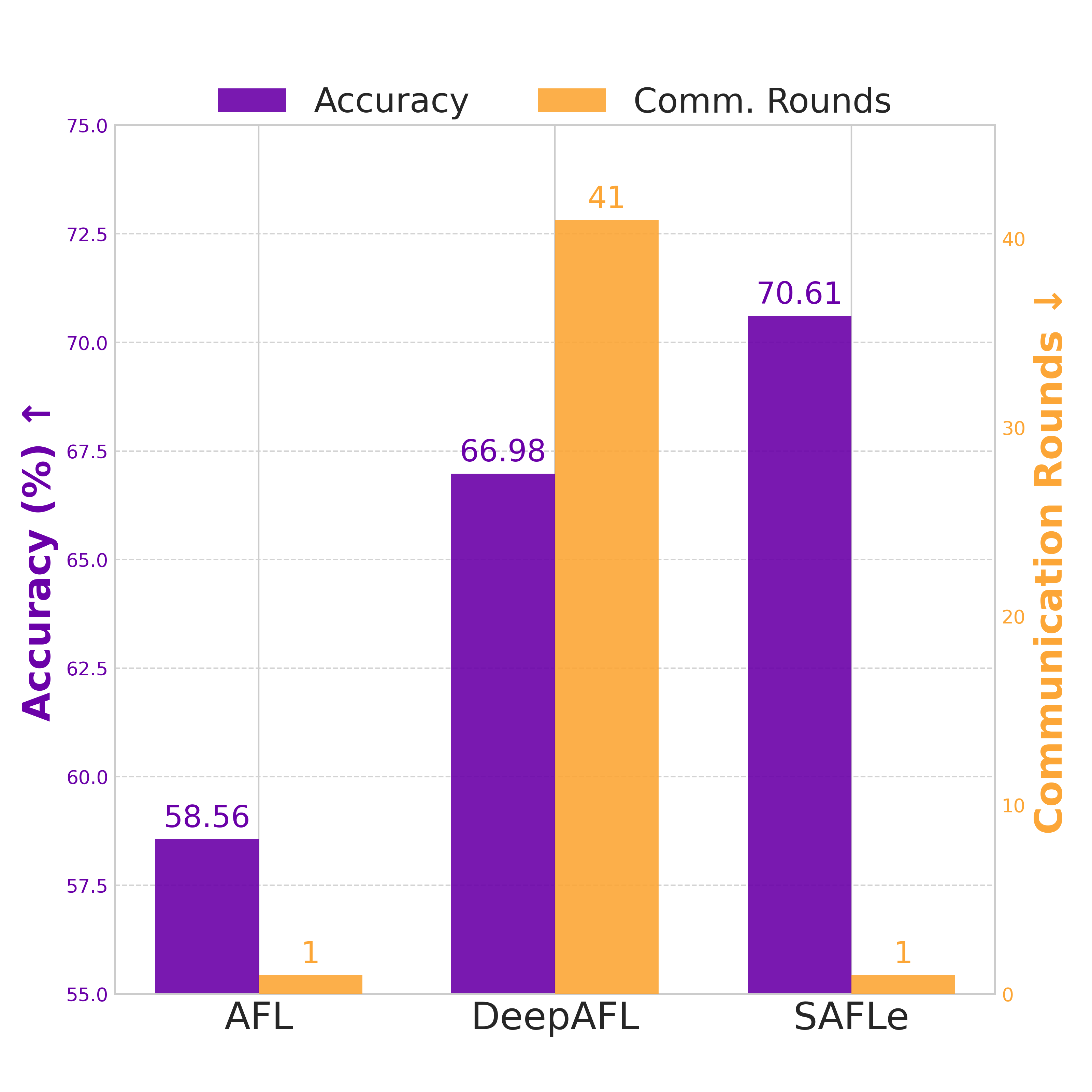}
        \caption{CIFAR-100}
        \label{fig:cifar100_clients}
    \end{subfigure}
    \hfill % Adds horizontal space between the figures
    % Subfigure (b) for Tiny-ImageNet
    \begin{subfigure}[b]{0.48\columnwidth}
        \includegraphics[width=\linewidth]{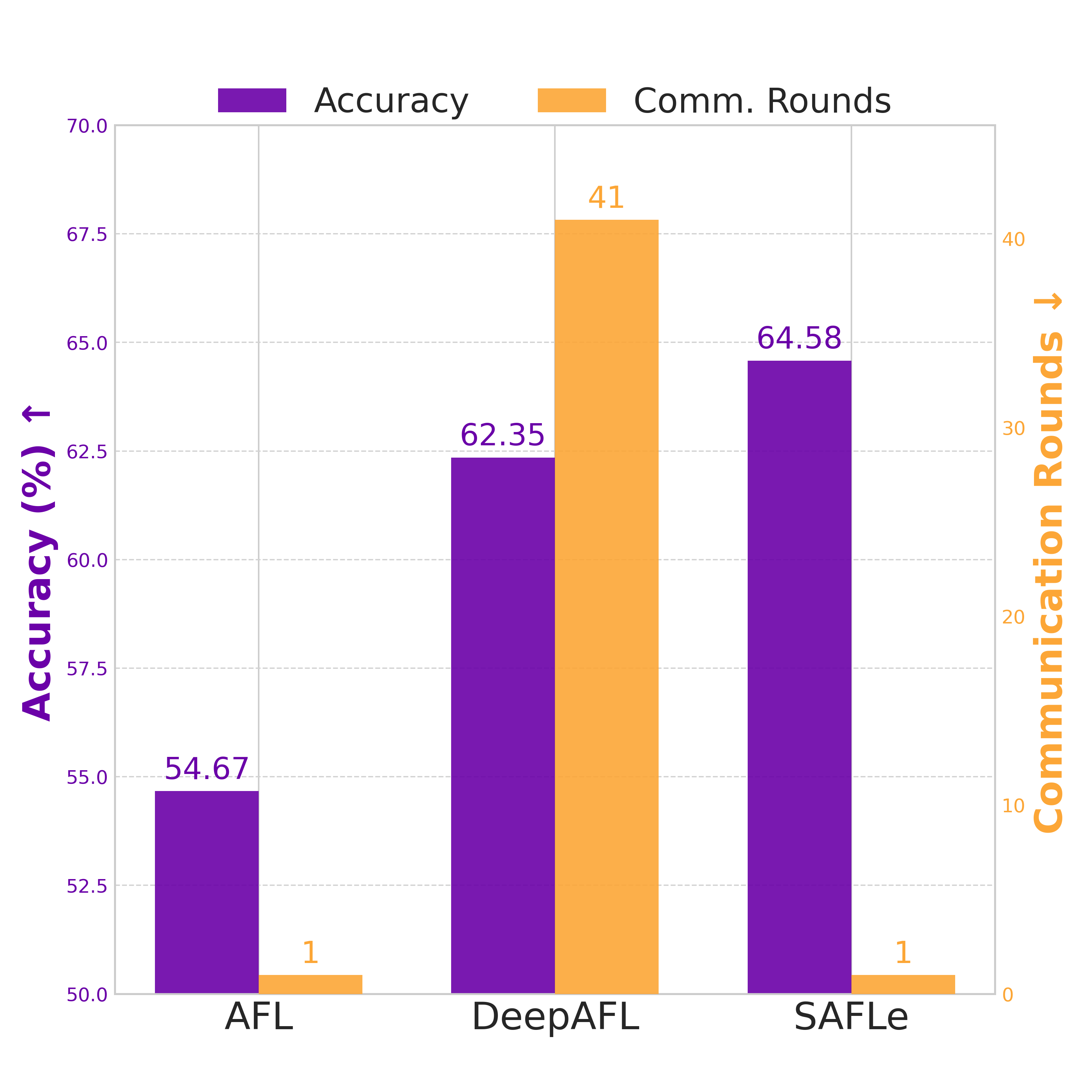}
        \caption{Tiny-ImageNet}
        \label{fig:tinyimagenet_clients}
    \end{subfigure}
    
    % Main figure caption
    \caption{Comparison of accuracy and communication rounds for analytic federated methods. SAFLe (ours) achieves the highest accuracy with the lowest rounds.}

    \label{fig:rounds}
\end{figure}

\begin{figure}[t]
\centering
    % Subfigure (a) for CIFAR-100
    \begin{subfigure}[b]{0.48\columnwidth}
        \includegraphics[width=\linewidth]{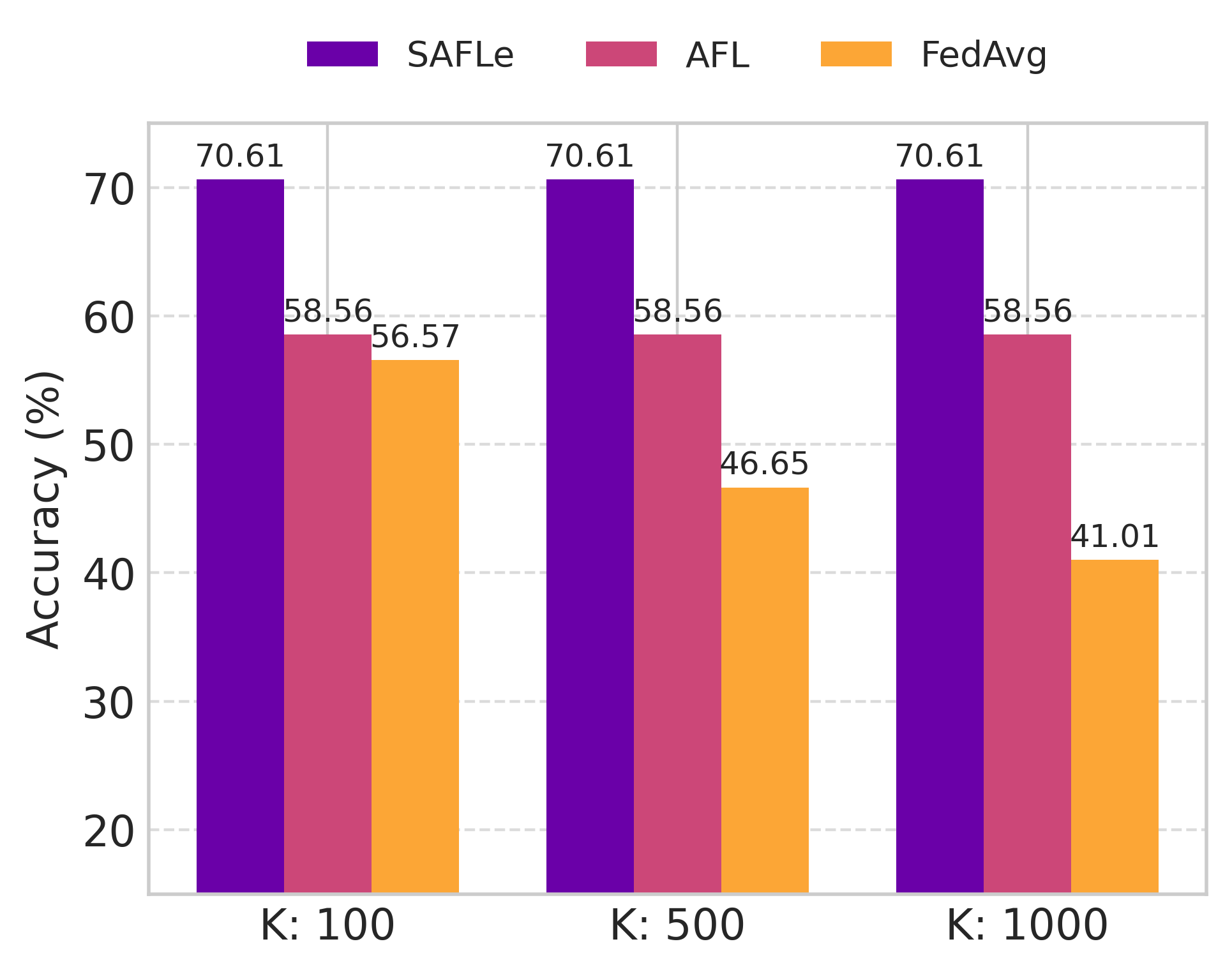}
        \caption{CIFAR-100}
        \label{fig:cifar100_clients}
    \end{subfigure}
    \hfill % Adds horizontal space between the figures
    % Subfigure (b) for Tiny-ImageNet
    \begin{subfigure}[b]{0.48\columnwidth}
        \includegraphics[width=\linewidth]{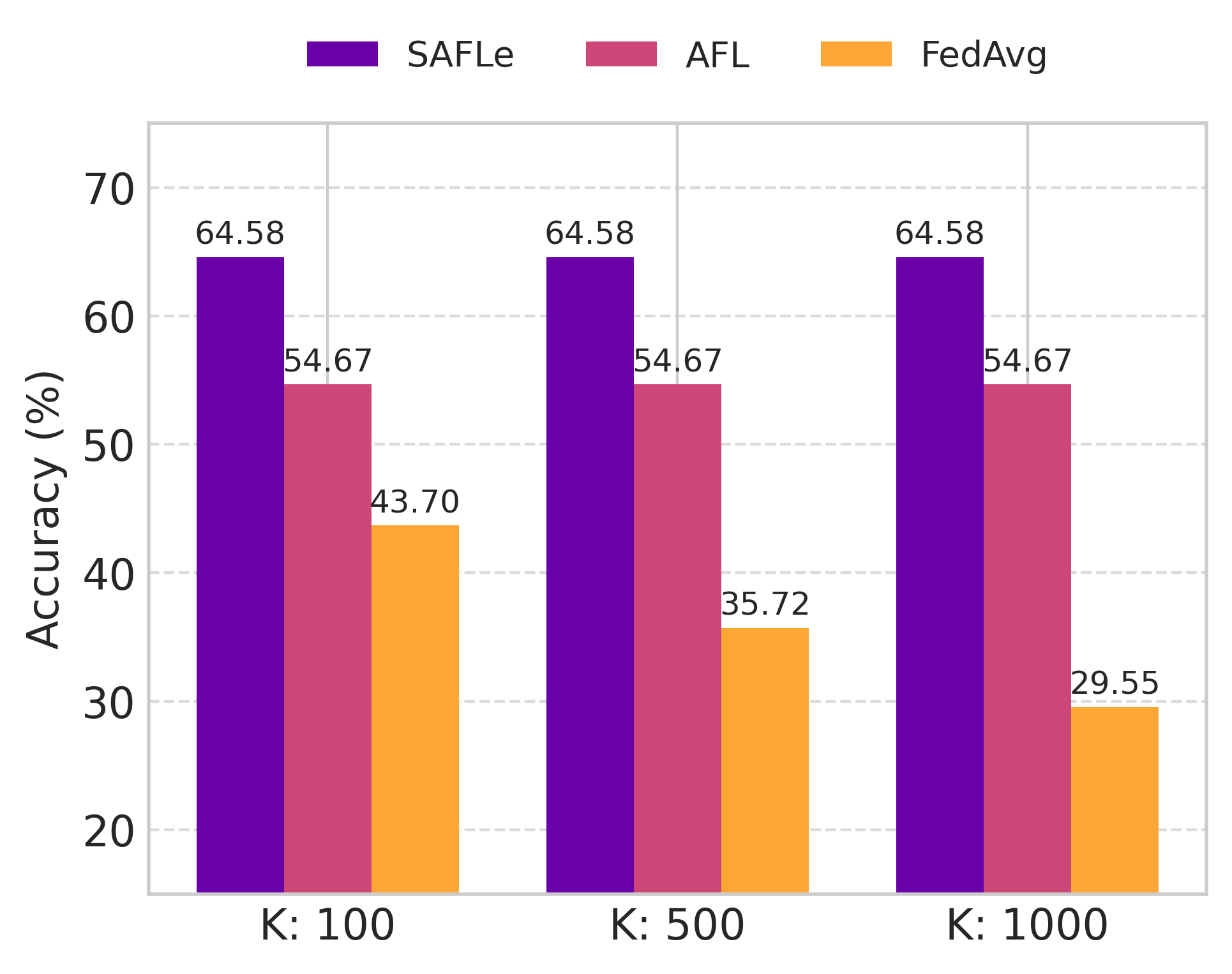}
        \caption{Tiny-ImageNet}
        \label{fig:tinyimagenet_clients}
    \end{subfigure}
    
    % Main figure caption
    \caption{Accuracy over increasing numbers of clients $(K)$. The proposed SAFLe is mathematically invariant to the number of clients, similar to AFL, whereas FedAvg's performance declines as $K$ increases.}
    \label{fig:client_number_comparison}
\end{figure}

% \textbf{Training Time Comparison.}
% All experiments were run on an NVIDIA RTX 4090 GPU, matching the hardware used in the prior analytic works \cite{He2025AFL, deepafl2025}. DeepAFL (Figure 3 in \cite{deepafl2025}) reports total training times of 130.7s, 145.7s, and 190.0s for its $T=5, 10, 20$ models on Tiny-ImageNet. To reach the same accuracy levels (60.3\%, 61.4\%, 62.4\%), our SAFLe models are computationally faster, taking 114s, 128s, and 143s, respectively.

\textbf{Ablation: Embedding Sparsity}
Our communication cost efficiency is enabled by a trade-off between accuracy and the sparsity of the communicated correlation matrix. Figure \ref{fig:emb_configs} demonstrates this through an ablation where, for a fixed model size, we vary the embedding layer configuration, that is, the number of embeddings $E$ and vocabulary size $V$. Configurations with a large group of small embeddings (high $E$, low $V$) achieve the best accuracy, which drops as we shift to fewer, larger embeddings (low $E$, high $V$). In contrast, the sparsity of the correlation matrix is inversely proportional. A layer with many small embeddings tends to be more dense, and the matrix quickly becomes very sparse as the embedding vocabulary size $V$ grows. The ideal trade-off is the point that balances high accuracy with high sparsity (low communication cost), which we found to be around a vocabulary size of $V \in [32, 64]$.

\begin{figure}[t]
\centering

% Subfigure (a) on top
\begin{subfigure}[b]{1.0\columnwidth}
    \centering
    \includegraphics[width=\linewidth]{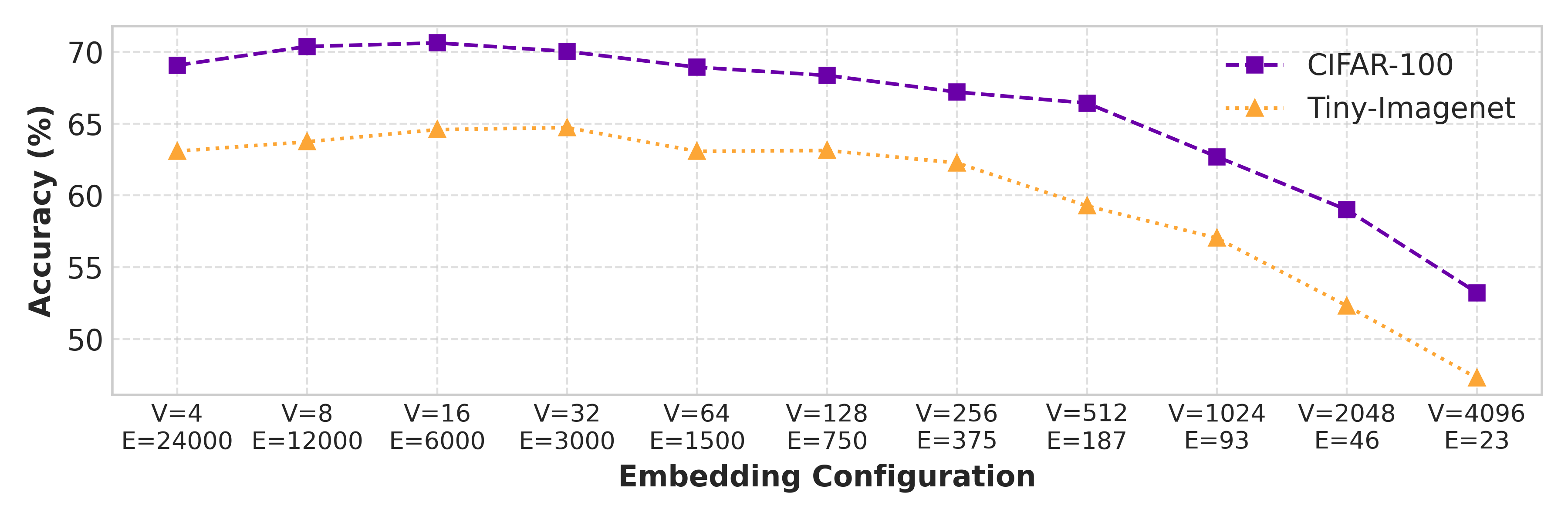}
    \caption{Embedding Configuration vs Accuracy}
    \label{fig:config_acc}
\end{subfigure}

\vspace{0.5em}

% Subfigure (b) below
\begin{subfigure}[b]{1.0\columnwidth}
    \centering
    \includegraphics[width=\linewidth]{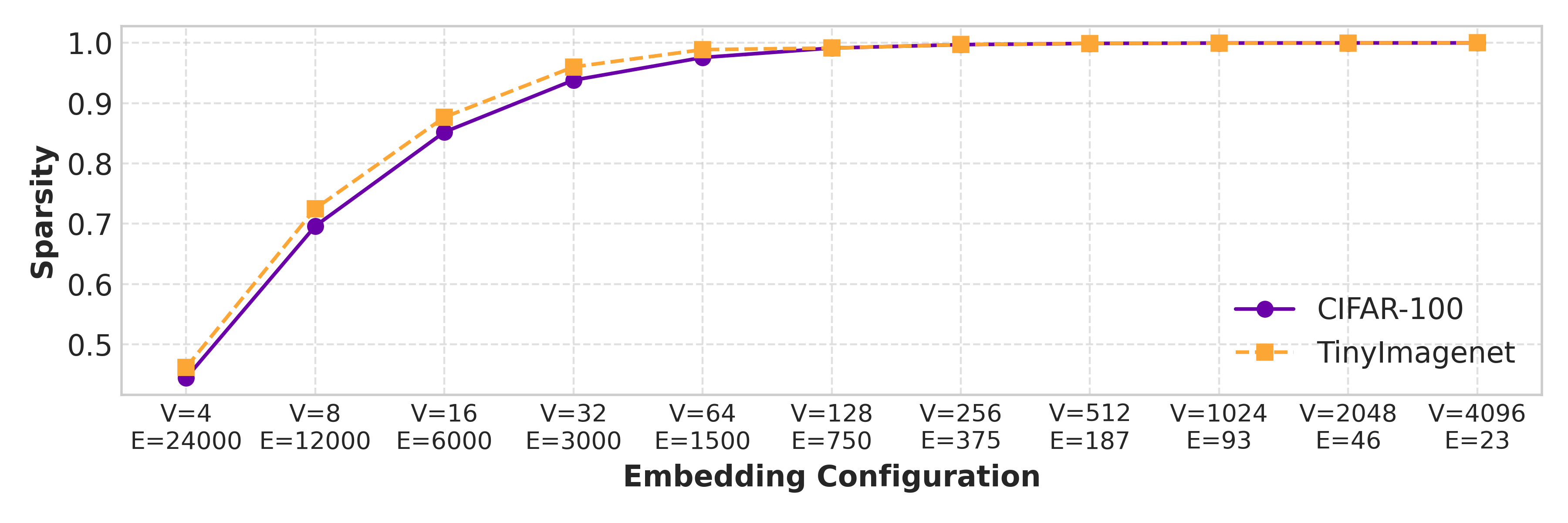}
    \caption{Embedding Configuration vs Sparsity}
    \label{fig:config_sparsity}
\end{subfigure}

\caption{Ablation study of embedding layer configurations (x-axis in $\log_2$ scale). We vary the number of embeddings ($E$) and their vocabulary size ($V$) while keeping total model size fixed. (a) Many small embeddings (high $E$, low $V$) generalize best, whereas fewer large embeddings (low $E$, high $V$) overfit and yield lower accuracy. (b) Many small embeddings produce dense correlation matrices with high communication cost, while fewer large embeddings result in sparser, more communication-efficient matrices.}
\label{fig:emb_configs}
\end{figure}

\textbf{Ablation: Bucketing Strategy}
The bucketing function used to discretize continuous backbone features is an important design choice. In Table \ref{tab:bucketing_strategy}, we compare three variants: Standard (Integer), One-Hot, and Binary Overlapping. Binary Overlapping performs best, while standard Integer quantization performs worst by a large margin. We believe this is because Integer quantization introduces a sharp “cliff effect,” where small changes across bin boundaries produce completely different indices and therefore very different embedding lookups, hurting generalization. One-Hot is slightly better but still assigns nearby inputs to disjoint representations. In contrast, Binary Overlapping preserves local similarity by giving nearby values similar bit patterns, allowing the model to learn more consistent embeddings and generalize better.

\begin{table}[t]
\footnotesize
\centering
\caption{Ablation study on bucketing strategies. Accuracy is reported across different numbers of buckets ($B_n$). Top accuracy in each setting is shown in green, and for each dataset in bold. 
%The Binary Overlap method consistently outperforms Integer and One-Hot quantization, as its design allows the model to generalize better to similar, nearby feature values.
}
\label{tab:bucketing_strategy}
% \begin{tabular}{l@{\hspace{1pt}}c@{\hspace{6pt}}ccc}
\begin{tabular}{lcccc}
\toprule
\textbf{Dataset} & $\mathbf{B_n}$ & \textbf{Integer} & \textbf{OneHot} & \textbf{Binary Overlap} \\
\midrule
\multirow{4}{*}{CIFAR-10}
& 2 & $87.16_{\pm0.00}$ & $87.54_{\pm0.27}$ & $\color{teal}88.43_{\pm0.12}$ \\
& 8 & $78.15_{\pm0.00}$ & $81.56_{\pm0.21}$ & $\color{teal}\mathbf{90.73_{\pm0.08}}$ \\
& 14 & $67.53_{\pm0.00}$ & $74.43_{\pm0.17}$ & $\color{teal}90.03_{\pm0.07}$ \\
& 20 & $59.75_{\pm0.00}$ & $68.27_{\pm0.15}$ & $\color{teal}89.12_{\pm0.07}$ \\
\midrule
\multirow{4}{*}{CIFAR-100} 
& 2 & $67.72_{\pm0.00}$ & $68.22_{\pm0.04}$ & $\color{teal}68.82_{\pm0.16}$ \\
& 8 & $58.28_{\pm0.00}$ & $61.66_{\pm0.04}$ & $\color{teal}\mathbf{70.61_{\pm0.12}}$ \\
& 14 & $47.75_{\pm0.00}$ & $54.65_{\pm0.04}$ & $\color{teal}70.12_{\pm0.14}$ \\
& 20 & $40.09_{\pm0.00}$ & $48.77_{\pm0.04}$ & $\color{teal}69.35_{\pm0.23}$ \\
\midrule
\multirow{4}{*}{\shortstack{Tiny \\ ImageNet}} 
& 2 & $61.38_{\pm0.00}$ & $61.78_{\pm0.01}$ & $\color{teal}62.43_{\pm0.27}$ \\
& 8 & $52.10_{\pm0.00}$ & $55.70_{\pm0.01}$ & $\color{teal}\mathbf{64.58_{\pm0.23}}$ \\
& 14 & $41.89_{\pm0.00}$ & $48.66_{\pm0.01}$ & $\color{teal}64.12_{\pm0.15}$ \\
& 20 & $34.70_{\pm0.00}$ & $43.22_{\pm0.01}$ & $\color{teal}63.93_{\pm0.21}$ \\
\bottomrule
\end{tabular}
\end{table}

\textbf{Ablation: Different Backbones.}
To show that SAFLe’s gains are not tied to a specific backbone, we follow the AFL ablation protocol \cite{He2025AFL} and compare against AFL using the same three frozen pre-trained backbones: ResNet-18 \cite{He2016ResNet}, VGG11 \cite{Simonyan2015VGG}, and ViT-B-16 \cite{Dosovitskiy2021ViT}. DeepAFL \cite{deepafl2025} is not included here because it does not report this experiment and its code is not yet public. As shown in Table \ref{tab:backbone_comparison_t}, SAFLe consistently outperforms AFL across all backbones and datasets. While stronger backbones improve both methods, SAFLe’s non-linear analytic head continues to provide a clear gain over AFL’s linear head, confirming that our advantage is robust to backbone choice.

Compared with recent one-shot federated methods built on strong pre-trained models, SAFLe remains superior. In particular, against TOFA \cite{tofa}, a recent training-free one-shot method based on a ViT-B/16 vision-language backbone, SAFLe achieves higher accuracy on both CIFAR-10 (95.31\% vs.\ 93.18\%) and CIFAR-100 (77.83\% vs.\ 76.63\%). This shows that SAFLe’s gains are not tied to a particular backbone family and remain robust even when compared against strong pre-trained one-shot baselines.

\begin{table}[h]
\centering
\caption{Ablation study of Top-1 accuracy (\%) using different backbones.}
\label{tab:backbone_comparison_t}
\resizebox{\columnwidth}{!}{% Add this wrapper to force-fit to column width
\begin{tabular}{l l c c c}
\toprule
\textbf{Dataset} & \textbf{Method} & \textbf{ResNet-18} & \textbf{VGG11} & \textbf{ViT-B-16} \\
\midrule
\multirow{2}{*}{CIFAR-10} & AFL & 80.75 & 82.72 & 93.92 \\
 & SAFLe & \textbf{90.73} & \textbf{87.50} & \textbf{95.31} \\
%\cmidrule(l){2-5} % Adds a light rule separating the dataset groups
\midrule
\multirow{2}{*}{CIFAR-100} & AFL & 58.56 & 60.43 & 75.45 \\
 & SAFLe & \textbf{70.61} & \textbf{62.68} & \textbf{77.83} \\
%\cmidrule(l){2-5}
\midrule
\multirow{2}{*}{Tiny-ImageNet} & AFL & 54.67 & 54.73 & 82.02 \\
 & SAFLe & \textbf{64.58} & \textbf{58.34} & \textbf{82.71} \\
\bottomrule
\end{tabular}
} % This is the closing brace for \resizebox
\end{table}

\textbf{Robustness to Label Noise.}
We also observe that SAFLe is reasonably robust to noisy labels, even though it does not use any explicit label-noise correction mechanism. On ISIC 2019~\cite{tschandl2018ham10000,codella2017isic,hernandez2024bcn20000} with label noise rates $\rho=\{0.2,0.3,0.4,0.6\}$, SAFLe achieves 69.47\%, 67.82\%, 66.53\%, and 62.94\%, respectively. These results are competitive with the specialized noise-robust method FedNoRo~\cite{wu2023fednoro}, which reports 69.8\%, 68.5\%, 68.9\%, and 63.3\%, and consistently outperform FedCorr+LA~\cite{xu2022fedcorr}, which reports 63.6\%, 61.9\%, 59.6\%, and 53.9\% at the same noise levels. This suggests that SAFLe inherits a useful degree of intrinsic robustness from its analytic formulation, even without task-specific noise handling.

\section{Conclusion}

In this work, we addressed the fundamental trade-off in analytic federated learning between single-round efficiency and non-linear expressive power. We proposed SAFLe, a framework that breaks this trade-off, achieving high model expressivity while retaining the single-shot, gradient-free properties of AFL.

Our method introduces a structured non-linear head using feature bucketing and sparse, grouped embeddings. We proved that this complex architecture is mathematically equivalent to a high-dimensional linear regression. This key theoretical insight allows SAFLe to be the first non-linear framework to be solved using AFL's single-round Absolute Aggregation law, ensuring mathematical invariance to data heterogeneity.

Our experiments demonstrate that SAFLe establishes a new state-of-the-art for analytic FL, outperforming both the linear AFL and the multi-round DeepAFL in accuracy across all benchmark datasets. By delivering high accuracy, true single-shot efficiency, and robustness to non-IID data, SAFLe provides a practical, efficient, and scalable solution for real-world federated vision tasks.

\section*{Acknowledgements}
This research was supported by Semiconductor Research Corporation (SRC) Task 3148.001, National Science Foundation (NSF) Grants \#2326894, \#2425655 (supported in part by the federal agency and Intel, Micron, Samsung, and Ericsson through the FuSe2 program), NVIDIA Applied Research Accelerator Program, and compute resources on the Vista GPU cluster through CGAI \& TACC at UT Austin. This work is supported in part by the NSF grant CCF-2531882 and a UT Cockrell School of
Engineering Doctoral Fellowship. Any opinions, findings, conclusions, or recommendations are those of the
authors and not of the funding agencies. 

\small
\bibliographystyle{ieeenat_fullname}
\bibliography{main}

% \clearpage
% \setcounter{page}{1}
% \maketitlesupplementary

% \section*{Appendix A}
% \label{sec:appendix_a}

\end{document}